\documentclass{article}
\usepackage{ijcai19}

\usepackage{times}
\usepackage{soul}
\usepackage{url}
\usepackage[hidelinks]{hyperref}
\usepackage[utf8]{inputenc}
\usepackage[small]{caption}
\usepackage{graphicx}
\usepackage{amsmath}
\usepackage{booktabs}
\usepackage{algorithm}
\usepackage{algorithmic}

\usepackage{tikz}
\usepackage{textcomp}
\usepackage{comment} 
\usepackage{amssymb}
\usepackage{amsmath}
\usepackage{amsthm}
\usepackage{xspace}
\usepackage{todonotes}
\usepackage{url}
\usepackage{paralist}
\usepackage{caption}
\usepackage{subcaption}
\usepackage{dsfont}
\usepackage{array,multirow}
\usepackage{ifthen}

\newcommand\blfootnote[1]{%
  \begingroup
  \renewcommand\thefootnote{}\footnote{#1}%
  \addtocounter{footnote}{-1}%
  \endgroup
}

\newcommand{\coms}{\xspace\scalebox{0.5}{\begin{tikzpicture}\draw[dash pattern=on 3pt off 3pt, line width = 1mm ] (0,0) -- (0.8,0); \node at (0,-0.08) {}; \end{tikzpicture}}\xspace}
\newcommand{\moves}{\rightarrow}

\newcommand{\addCommunicationB}[1]
{
	\node[above right = 0.5mm and 4mm of #1, inner sep=0mm] (nb2) {B};
	\draw[tobehere,communication] (#1) -- (nb2);
}

\newcommand\setnodes{V}
\newcommand\basenode{B}
\newcommand{\sourcenode}{s}
\newcommand{\targetnode}{t}
\newcommand\USTCONN{USTCONN}
\newcommand\UCONN{UCONN}
\newcommand{\myundirectedgraphbounded}{\mathfrak G}
\newcommand{\myundirectedgraph}{\mathfrak G'}

\newcommand\pb[2]{${#1}_{#2}$~}

\newcommand\pbReach[1]{\pb{Reachability}{#1}}
\newcommand\pbBReach[1]{\pb{bReachability}{#1}}
\newcommand\pbCoverage[1]{\pb{Coverage}{#1}}
\newcommand\pbBCoverage[1]{\pb{bCoverage}{#1}}

\newcommand{\idDir}{dir}	% identifier for directed topological graph
\newcommand{\idNC}{nc}		% identifier for neighbor-communicable topological graph
\newcommand{\idSM}{sm}		% identifier for sight-moveable topological graph
\newcommand{\idCC}{cc}		% identifier for complete-communication topological graph
\newcommand{\pbReachDir}{\pbReach{\idDir}}
\newcommand{\pbBReachDir}{\pbBReach{\idDir}}
\newcommand{\pbCoverageDir}{\pbCoverage{\idDir}}
\newcommand{\pbBCoverageDir}{\pbBCoverage{\idDir}}

\newcommand{\pbReachNC}{\pbReach{\idNC}}

\newcommand{\pbCoverageNC}{\pbCoverage{\idNC}}

\newcommand{\pbReachSM}{\pbReach{\idSM}}
\newcommand{\pbBReachSM}{\pbBReach{\idSM}}
\newcommand{\pbCoverageSM}{\pbCoverage{\idSM}}

\newcommand{\pbBReachCC}{\pbBReach{\idCC}}

\newcommand{\pbBCoverageCC}{\pbBCoverage{\idCC}}

\newtheorem{theorem}{Theorem}
\newtheorem{fact}[theorem]{Fact}
\newtheorem{definition}[theorem]{Definition}
\newtheorem{proposition}[theorem]{Proposition}
\newtheorem{lemma}[theorem]{Lemma}

\tikzstyle{copyofG} = [fill=gray!20!white, draw=none, rounded corners=8pt]
\tikzstyle{communication} = [dash pattern=on 1pt off 1pt] %[line cap=rect, dot diameter=1pt, dot spacing=2pt, dots] %[dash pattern=on 2pt off 2pt on \the\pgflinewidth off 2pt]
\tikzstyle{tobehere} = [color=blue!50!white]
\newcommand\set[1]{\{#1\}}

\newcommand{\booleanvariablegeneric}{{x}}
\newcommand{\clausegeneric}{\mathsf{c}}

\newcommand{\ok}{\includegraphics[height=2mm]{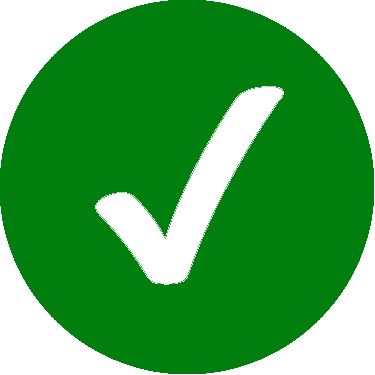}}
\newcommand{\imgDrone}{\includegraphics[height=2mm]{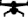}}

\usetikzlibrary{decorations.pathmorphing}
\usetikzlibrary{decorations.pathreplacing}
\usetikzlibrary{arrows}
\usetikzlibrary{positioning}

\makeatletter
\tikzset{
	dot diameter/.store in=\dot@diameter,
	dot diameter=3pt,
	dot spacing/.store in=\dot@spacing,
	dot spacing=10pt,
	dots/.style={
		line width=\dot@diameter,
		line cap=round,
		dash pattern=on 0pt off \dot@spacing
	}
}
\makeatother

\newif\iffull
\fulltrue

\newif\ifauthor
\authortrue

{

\title{Reachability and Coverage Planning for Connected Agents:
		\large
		\\
		Extended Version}
\ifauthor
\author{
Tristan Charrier$^1$
\and
Arthur Queffelec$^1$\and
Ocan Sankur$^2$\And
Fran\c{c}ois Schwarzentruber$^1$
\affiliations
Univ Rennes, CNRS, IRISA$^1$\\
Univ Rennes, Inria, CNRS, IRISA$^2$\\
\emails
firstname.lastname@irisa.fr$^1$,
ocan.sankur@inria.fr$^2$,
}
\else
\author{\# 4601
}
\fi

\begin{document}

\maketitle

\begin{abstract}
	Motivated by the increasing appeal of robots in information-gathering missions,
	we study multi-agent path planning problems in which the agents must remain
	interconnected. We model an area by a topological graph specifying the movement
	and the connectivity constraints of the agents. We study the theoretical
	complexity of the reachability and the coverage problems of a fleet of connected
	agents on various classes of topological graphs.
	We establish the complexity of these problems on known classes, and introduce a new class
        called \emph{sight-moveable graphs}
        which admit efficient algorithms.
\end{abstract}

%  \begin{document}

\section{Introduction}
\label{sec:intro}

A number of use cases of planning rose in information-gathering missions from
the development of unmanned autonomous vehicles (UAVs). For instance, in search
and rescue missions, a fleet of drones can cover a lot of ground in a short
amount of time and report any finding to a mission supervisor to narrow the
search for the rescue team. Other examples are the analysis of terrain for smart
farms and in hazardous locations. For this kind of missions, the information
gathered is used for decision making at a supervising station. Thus, the drones
need to be constantly in communication with the station to report the gathered information
during the mission. The use of multiple UAVs to cover an area not only reduces the time
required to complete the mission but can also enable reaching locations which would
not be reachable with a single drone due to connection constraints.
%be used to reach some location,
%where a single one cannot, without losing the connection with the supervisor.
\iffull
\else
\blfootnote{$^1$\href{https://drive.google.com/open?id=1dmvAWjN6vNwxGIBw9TmXCOjw3gKCZWoh}{https://drive.google.com/open?id=1dmvAWjN6vNwxGIBw9T\newline mXCOjw3gKCZWoh}}
\fi
The original multi-agent path planning problem asks for a plan to reach a
configuration of agents in a graph. However, an important problem for search and
rescue missions or terrain analysis is the coverage of an area. We thus study
both the coverage problem and the reachability problem under a connection
constraint over the agents which requires them to be connected to the base
either directly or via another agent, who can relay its data. We establish the
computational complexity of the connected coverage in its general case and for a
practical subclass introduced recently~\cite{dblp:conf/aaai/tateobrab18} in
which the UAVs can communicate with others located within one step, called the
\emph{neighbor-communicable} topological graphs. We show that the coverage is
PSPACE-complete in the general case, and remains so for neighbor-communicable
topological graphs. Thus, restricting to neighbor-communicable graphs does not
render the problem feasible, and the relatively high complexity unfortunately
remains. Note that this is in line with the PSPACE-completeness of the
reachability problem recently reported in~\cite{dblp:conf/aaai/tateobrab18}.
%Note that the PSPACE-completeness of the reachability problem was recently shown~

% We show that, in the
% general case, the complexity of deciding whether the coverage is achievable is
% PSPACE-complete and that this complexity result holds on the
% neighbor-communicable topological graphs.
%\textbf{\textcolor{red}{TODO: CITE RETINA}}

Our main result in this paper is the definition of a class of topological graphs
which is well adapted and realistic for UAV missions, and for which the coverage
and reachability problems admit efficient solutions. Our subclass, called
\emph{sight-moveable} graphs, is defined assuming that the UAVs cannot
communicate through obstacles and are restricted to line-of-sight communication.
This class emerged from an ongoing case study for a drone assisted search and
rescue project in which the authors take part. For this class, we prove that
both the reachability and coverage problems are in LOGSPACE while the existence
of a bounded execution is in NP. This drastically changes the status of this
problem since by LOGSPACE $\subseteq$ NC (this is the class of problems solvable
in polylogarithmic time in a parallel machine with a polynomial number of
processors),
%\footnote{\todo[inline]{Explain NC !}}, 
%the LOGSPACE
%membership means that
one can build an efficient parallel algorithm~\cite{Cook:1979}.
The NP upper bound is also useful since this means efficient SAT solvers can be used directly
to compute bounded executions.
We prove that our algorithms for the bounded variants are optimal by showing NP-hardness
in each case.

%The motivation behind our subclass is the following.
%In some applications, the UAVs gather a large amount of data for a rigorous
%analysis. Hence, the connection between the UAVs and the supervising station
%needs to be reliable.
%the sight-moveable topological graphs.
%We show that deciding the existence of a
%feasible plan can be achieved by a logarithmic-space algorithm and that deciding
%of the existence of a bounded execution is NP-complete. 
%In addition to the
%reachability of a configuration, we study the problem of bounded and unbounded
%coverage on sight-moveable topological graphs. We claim that the unbounded
%version is in LOGSPACE, and the bounded version is NP-complete. 

\input{picture_execution}

In this work, we consider all agents to be anonymous in both cases (reachability
and coverage). Furthermore, we consider the collisions to be handled by the
agents themselves, hence are not considered along the results of this paper. We
depicted a covering execution of a topological graph by 3 UAVs in Figure
\ref{figure:exampleexecution}. In this example, the UAVs need to gather
information at each node of the graph while staying connected to the base (red
node) during the whole mission.

%The rest of the paper is structured as follows. 
In Section
\ref{sec:preli}, we present the typical notions used in MAPP and their extension
for our case and the known results in connected planning. In Sections
\ref{sec:dir} to \ref{sec:compl}, we study the complexity of our problems from the general case
to the most restrictive one. We describe the related works in Section \ref{sec:relate}.
%Finally, 
We conclude in Section \ref{sec:concl}.

\iffull
\else
\mbox{The full version of the paper\footnotemark[1] is available.}
\fi

%%% Local Variables:
%%% mode: latex
%%% TeX-master: "main"
%%% End:

\section{Preliminaries}
\label{sec:preli}

%In this section, we define the required notions used throughout the work.

We first present the topological graphs and the subclasses we consider on which
we study the complexity of our problems. Then, we give definitions of plans and
executions, and formally define problems we consider.

In most applications of path planning, the space is discretized
in order to generate a graph of movements on which algorithms are executed.
For instance, \emph{regular grids} which decompose the space
in square, triangular or hexagonal cells, \emph{irregular grids} with techniques
such as quadtree \cite{Finkel:1974, Knoll:06} or \emph{Vorono\"i diagram}
comprehensively discussed in the survey \cite{Aurenhammer:1991}.

Our work is independent of the particular method used to obtain the discretization.
We only work under the hypothesis that a feasible plan on the graph generated by 
the discretization is also feasible in the continuous space.
%The multi-agent path planning (MAPP) is a well studied problem. Typically, MAPP
% starts by discretizing the area into a graph on which the planning is computed.
% Many discretizing methods as being developed throughout the study of path
% planning and MAPP. For instance, \emph{regular grids} which decompose the space
% in square, triangular or hexagonal cells, \emph{irregular grids} with techniques
% like quadtree \cite{Finkel:1974, Knoll:06} or \emph{Vorono\"i diagram}
% comprehensively discussed in the survey \cite{Aurenhammer:1991}.

% Even though those methods have their advantages and disadvantages, we do not
% consider a particular method. We work under the hypothesis that a feasible plan
% on the graph generated by the discretization is also feasible in the continuous
% area.

\subsection{Topological graph}
\label{preli:graph}

\iffull
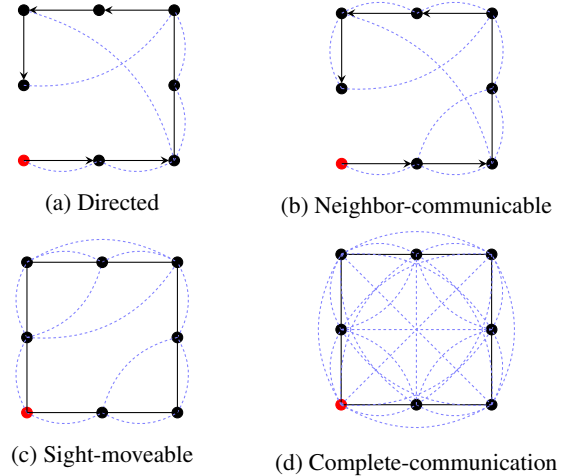
\begin{figure}
      \centering
      \begin{subfigure}{.23\textwidth}
        \centering
          \begin{tikzpicture}
            \tikzstyle{node} = [draw, circle, fill=black, inner sep=0.5mm];
            \tikzstyle{basenode} = [draw, circle, red, fill=red, inner sep=0.5mm];
            \foreach \x in {0,...,2}{
              \foreach \y in {0,...,2}{
                \ifthenelse{\x=0 \AND \y=0}{
                  \node [basenode] (\x,\y) at (\x,\y) {};
                }{
                  \ifthenelse{\x=1 \AND \y=1}{}{\node [node] (\x,\y) at (\x,\y) {};}
                }
              }
            }
            
            \foreach \x in {0,1}{
              \draw[communication, tobehere] (\x,0) edge[bend right] (\x+1,0);
              \node (0\x) at (\x+1,0) {};
              \draw[arrows={stealth-}] (0\x.center)+(-0.7mm,0) -- (\x,0);
            }
    
            \foreach \x in {0,1}{
              % \draw[communication, tobehere] (\x,2) edge[bend left] (\x+1,2);
              \node (2\x) at (\x,2) {};
              \draw[arrows={stealth-}] (2\x.center)+(0.7mm,0) -- (\x+1,2);
            }
    
            % \draw[communication, tobehere] (0,1) edge[bend left] (0,2);
            \node (a) at (0,1) {};
            \draw[arrows={stealth-}] (a.center)+(0,0.7mm) -- (0,2);
    
            \foreach \y in {0,1}{
              \draw[communication, tobehere] (2,\y) edge[bend right] (2,\y+1);
              \draw[-] (2,\y) -- (2,\y+1);
            }
    
            \draw[communication, tobehere] (2,0) edge[bend right] (0,2);
            % \draw[communication, tobehere] (2,1) edge[bend right] (1,0);
            \draw[communication, tobehere] (0,1) edge[bend right] (2,2);
          \end{tikzpicture}
        \caption{\footnotesize{Directed}}
        \label{fig:topo:dir}
      \end{subfigure}
      \begin{subfigure}{.23\textwidth}
        \centering
          \begin{tikzpicture}
            \tikzstyle{node} = [draw, circle, fill=black, inner sep=0.5mm];
            \tikzstyle{basenode} = [draw, circle, red, fill=red, inner sep=0.5mm];
            \foreach \x in {0,...,2}{
              \foreach \y in {0,...,2}{
                \ifthenelse{\x=0 \AND \y=0}{
                  \node [basenode] (\x,\y) at (\x,\y) {};
                }{
                  \ifthenelse{\x=1 \AND \y=1}{}{\node [node] (\x,\y) at (\x,\y) {};}
                }
              }
            }
            
            \foreach \x in {0,1}{
              \draw[communication, tobehere] (\x,0) edge[bend right] (\x+1,0);
              \node (0\x) at (\x+1,0) {};
              \draw[arrows={stealth-}] (0\x.center)+(-0.7mm,0) -- (\x,0);
            }
    
            \foreach \x in {0,1}{
              \draw[communication, tobehere] (\x,2) edge[bend left] (\x+1,2);
              \node (2\x) at (\x,2) {};
              \draw[arrows={stealth-}] (2\x.center)+(0.7mm,0) -- (\x+1,2);
            }
    
            \draw[communication, tobehere] (0,1) edge[bend left] (0,2);
            \node (a) at (0,1) {};
            \draw[arrows={stealth-}] (a.center)+(0,0.7mm) -- (0,2);
    
            \foreach \y in {0,1}{
              \draw[communication, tobehere] (2,\y) edge[bend right] (2,\y+1);
              \draw[-] (2,\y) -- (2,\y+1);
            }
    
            \draw[communication, tobehere] (2,0) edge[bend right] (0,2);
            \draw[communication, tobehere] (2,1) edge[bend right] (1,0);
            \draw[communication, tobehere] (0,1) edge[bend right] (2,2);
          \end{tikzpicture}
        \caption{\footnotesize{Neighbor-communicable}}
        \label{fig:topo:nc}
      \end{subfigure}
      \begin{subfigure}{.23\textwidth}
        \centering
          \begin{tikzpicture}
            \tikzstyle{node} = [draw, circle, fill=black, inner sep=0.5mm];
            \tikzstyle{basenode} = [draw, circle, red, fill=red, inner sep=0.5mm];
            \foreach \x in {0,...,2}{
              \foreach \y in {0,...,2}{
                \ifthenelse{\x=0 \AND \y=0}{
                  \node [basenode] (\x,\y) at (\x,\y) {};
                }{
                  \ifthenelse{\x=1 \AND \y=1}{}{\node [node] (\x,\y) at (\x,\y) {};}
                }
              }
            }
            
            \foreach \x in {0,1}{
              \draw[communication, tobehere] (\x,0) edge[bend right] (\x+1,0);
              \draw[-] (\x,0) -- (\x+1,0);
            }
    
            \foreach \x in {0,1}{
              \draw[communication, tobehere] (\x,2) edge[bend left] (\x+1,2);
              \draw[-] (\x,2) -- (\x+1,2);
            }
    
            \foreach \y in {0,1}{
              \draw[communication, tobehere] (0,\y) edge[bend left] (0,\y+1);
              \draw[-] (0,\y) -- (0,\y+1);
            }
    
            \foreach \y in {0,1}{
              \draw[communication, tobehere] (2,\y) edge[bend right] (2,\y+1);
              \draw[-] (2,\y) -- (2,\y+1);
            }
    
            \draw[communication, tobehere] (2,1) edge[bend right] (1,0);
    
            \draw[communication, tobehere] (0,1) edge[bend right] (2,2);
            \draw[communication, tobehere] (0,1) edge[bend right] (1,2);
            \draw[communication, tobehere] (0,2) edge[bend left] (2,2);
          \end{tikzpicture}
        
        \caption{\footnotesize{Sight-moveable}}
        \label{fig:topo:sm}
      \end{subfigure}
      \begin{subfigure}{.23\textwidth}
        \centering
          \begin{tikzpicture}
            \tikzstyle{node} = [draw, circle, fill=black, inner sep=0.5mm];
            \tikzstyle{basenode} = [draw, circle, red, fill=red, inner sep=0.5mm];
            \foreach \x in {0,...,2}{
              \foreach \y in {0,...,2}{
                \ifthenelse{\x=0 \AND \y=0}{
                  \node [basenode] (\x,\y) at (\x,\y) {};
                }{
                  \ifthenelse{\x=1 \AND \y=1}{}{\node [node] (\x,\y) at (\x,\y) {};}
                }
              }
            }
            
            \foreach \x in {0,1}{
              \draw[communication, tobehere] (\x,0) edge[bend right] (\x+1,0);
              \draw[-] (\x,0) -- (\x+1,0);
            }
    
            \foreach \x in {0,1}{
              \draw[communication, tobehere] (\x,2) edge[bend left] (\x+1,2);
              \draw[-] (\x,2) -- (\x+1,2);
            }
    
            \foreach \y in {0,1}{
              \draw[communication, tobehere] (0,\y) edge[bend left] (0,\y+1);
              \draw[-] (0,\y) -- (0,\y+1);
            }
    
            \foreach \y in {0,1}{
              \draw[communication, tobehere] (2,\y) edge[bend right] (2,\y+1);
              \draw[-] (2,\y) -- (2,\y+1);
            }
    
            \draw[communication, tobehere] (2,1) edge[bend right] (1,0);
            \draw[communication, tobehere] (0,1) edge[bend right] (1,2);
            \draw[communication, tobehere] (0,1) edge[bend left] (1,0);
            \draw[communication, tobehere] (2,1) edge[bend left] (1,2);
    
            \draw[communication, tobehere] (0,0) edge[bend left] (0,2);
            \draw[communication, tobehere] (0,0) edge[bend right] (2,0);
            \draw[communication, tobehere] (0,2) edge[bend left] (2,2);
            \draw[communication, tobehere] (2,0) edge[bend right] (2,2);
    
            \draw[communication, tobehere] (0,0) edge (2,2);
            \draw[communication, tobehere] (0,2) edge (2,0);
    
            \draw[communication, tobehere] (1,0) edge (1,2);
            \draw[communication, tobehere] (0,1) edge (2,1);
    
            \draw[communication, tobehere] (0,0) edge [bend left] (1,2);
            \draw[communication, tobehere] (0,0) edge [bend right] (2,1);
    
            \draw[communication, tobehere] (2,2) edge [bend left] (1,0);
            \draw[communication, tobehere] (2,2) edge [bend right] (0,1);
    
            \draw[communication, tobehere] (0,2) edge [bend right] (1,0);
            \draw[communication, tobehere] (0,2) edge [bend left] (2,1);
    
            \draw[communication, tobehere] (2,0) edge [bend left] (0,1);
            \draw[communication, tobehere] (2,0) edge [bend right] (1,2);

          \end{tikzpicture}
        
        \caption{\footnotesize{Complete-communication}}
        \label{fig:topo:compl}
      \end{subfigure}
      \caption{Examples of topological graphs.}
    \end{figure}
\fi
Compared to the graphs used in MAPP, we also consider \emph{communication edges}
which specify whether agents at two different locations can communicate. We call
graphs with this additional information \emph{topological graphs}. The formal
definition is the following.

\begin{definition}[Topological graph] A topological graph is a tuple $G=\langle
    \setnodes, \moves, \coms \rangle$, with $V$ a finite set of nodes containing
    a distinguished element $\basenode$, $\moves\subseteq V \times V$ a set of
    movement edges and $\coms\subseteq V\times V$ a set of undirected
    communication edges.
    \label{def:topo:dir}
\end{definition}

The node $\basenode$ symbolizes the supervision station from which the agents
start the mission. A topological graph is undirected if~$\langle\setnodes,
\moves\rangle$ is an undirected graph.

We will now consider three subclasses of interest.

In most situations, if an agent can move to a location in one step, it can also
communicate with an agent at that location. This class has been discussed in
\cite{dblp:conf/aaai/tateobrab18}. We call topological graphs satisfying this
requirement \emph{neighbor-communicable}. An example is given in
Figure~\ref{fig:topo:nc}.

\begin{definition}[Neighbor-Communicable topological graph] A
  neighbor-communicable topological graph is a topological graph such that $v\moves
  v'$ implies $v\coms v'$.
  \label{def:topo:nc}
\end{definition}

Another class of graphs is that of \emph{sight-moveable} and is the main one for
which we give efficient algorithms. First, this class requires the movement
edges to be undirected and reflexive. Second, whenever an agent can communicate
with another node, then it can also move to that node while maintaining the communication.
This intuitively means that we disallow communication through obstacles.
The formal
definition follows, and an example is depicted in Figure~\ref{fig:topo:sm}.
% Another class of graphs is that of \emph{sight-moveable} and is the main one for
% which we give efficient algorithms. First, this class requires the movement
% edges to be undirected and reflexive. Second, we consider the case where agents
% are not able to communicate with other agents through obstacles but only if they
% are in line of sight. It is indeed natural to consider that, for example,
% if a drone is at the base it can go to any node which communicates with the base
% in a straight line and not losing the communication on its way. The formal
% definition follows, and an example is depicted in Figure~\ref{fig:topo:sm}.

\begin{definition}[Sight-Moveable topological graph] A sight-moveable topologic
  graph is an undirected neighbor-communicable topological graph in
  which for all $v \in V$, $v\moves v$ and whenever $v \coms v'$, there exists a
  sequence $\rho=\langle \rho_1, \dots, \rho_n\rangle$ of nodes such that
  $v=\rho_1$, $v'=\rho_n$, $v \coms \rho_i$ and $\rho_i \moves \rho_{i+1}$ for
  all $i\in\{1,n\}$.
  \label{def:topo:sm}
\end{definition}

Last, we define the \emph{complete-communication} topological graphs which are
simply sight-moveable topological graphs with a complete communication topology.
This subclass can model the use of a hovering connected agent which allows a
constant communication for all the agents in the area. An example of such a
graph is depicted in Figure~\ref{fig:topo:compl}, and the formal definition is
the following.

\begin{definition}[Complete-Communication topological graph] A
  complete-communication topological graph is a sight-moveable topological
  graph such that $\coms=V\times V$. \label{def:topo:compl}
\end{definition}
Observe that complete-communication graphs are reflexive, undirected, connected
graphs with $\coms=V\times V$.

\subsection{Execution}
\label{preli:exec}

A plan or \emph{execution}, in MAPP, is a list of \emph{configurations} which
describes the placement of the agents during the mission. The formal definition
of a configuration is the following.

\begin{definition}[Configuration] A configuration $c$ of $n$ agents in a
  topological graph $G$ is an element of $V^n$ denoted $c=\langle c_1, \dots,
  c_n\rangle$ in which $c_i$ is the location of the agent $i$ such
  that the graph $\langle V_a, \coms \cap V_a \times V_a\rangle$ is connected
  with $V_a = \{\basenode, c_1, \dots, c_n\}$. We extend our notation 
  and denote $c \moves c'$ when $c_i \moves c'_i$ for all $0<i\leq n$.
  \label{def:config}
\end{definition}

Furthermore, in the literature, MAPP asks to associate an agent to a specific
goal. However, given that we are interested in covering an area with a fleet of
agents, the anonymity is useful to get more efficient plans.

\textbf{Anonymity.} In the rest of this paper, we consider the agents to be
\emph{anonymous}. In other words, a configuration $c$ is equivalent to a
configuration $c'$ iff $c$ is a reordering of $c'$.

%\textbf{\textcolor{red}{TODO: Definition of the anonimity enough ??}
%Ocan: I think this is sufficient}

Moreover, an important notion in MAPP is the computation of collision-free
plans. In the drone case, in which we are particularly interested, one can place
drones at different heights to avoid collisions. Additionally, most drones are,
nowadays, equipped with local collision avoidance systems.

\textbf{Collisions.} We do not deal with meet- or head-on-collisions
of agents, \textit{i.e.} we allow two agents to be located
in a same node,
and to move in opposite directions of an edge within a step.

% \begin{definition}[Execution] An execution $e$ of length $\ell$ with $n$ agents
%     in a graph $G$ is a sequence of configuration $\langle c^1, \dots,
%     c^\ell\rangle$ such that for $c^i \moves c^{i+1}$ for all $0<i<\ell$.
% \end{definition}

% \begin{definition}[Covering Execution] A covering execution $e=\langle c^1,
%     \dots c^\ell\rangle$ of length $\ell$ with $n$ agents in a graph $G$ is an
%     execution such that $c^1=c^\ell=\langle B, \dots, B\rangle$ and for all
%     $v\in V$, there exists $i\in \{1,\ldots,\ell\}$ with $v\in c^i$.
% \end{definition}
An \emph{execution} $e$ of length $\ell$ with $n$ agents
    in a graph $G$ is a sequence of configuration $\langle c^1, \dots,
    c^\ell\rangle$ such that for $c^i \moves c^{i+1}$ for all $0<i<\ell$.
%\end{definition}
%\begin{definition}[Covering Execution] 
A \emph{covering execution} $e=\langle c^1,
    \dots c^\ell\rangle$ of length $\ell$ with $n$ agents in a graph $G$ is an
    execution such that $c^1=c^\ell=\langle B, \dots, B\rangle$ and for all
    $v\in V$, there exists $i\in \{1,\ldots,\ell\}$ with $v\in c^i$.
%\end{definition}

\subsection{Decision problems}
\label{preli:prob}

We define the MAPP problems, the \emph{Reachability} problem along with its
bounded version, \emph{bReachability}, for the makespan optimization of the
plan. In addition, we define the \emph{Coverage} problem and the bounded
coverage, \emph{bCoverage}.

\begin{definition}[Reachability]
  Given a topological graph $G$ and a configuration~$c$,
  decide if there is an execution $\langle c^1, \dots,
    c^\ell\rangle$ in $G$ such that $c^1 = \langle B, \dots, B\rangle$ and
    $c^\ell = c$.
    % \end{itemize}
  \label{def:pb:reach}
\end{definition}
% \begin{definition}[Reachability]~
%   \begin{itemize}
%   \item[] Input: a topological graph $G$ and $c$ a configuration;
%   \item[] Output: does there exist an execution $\langle c^1, \dots,
%     c^\ell\rangle$ in $G$ such that $c^1 = \langle B, \dots, B\rangle$ and
%     $c^\ell = c$?
%   \end{itemize}
%   \label{def:pb:reach}
% \end{definition}

\begin{definition}[Coverage]
  Given a topological graph $G$ and $n\in \mathds{N}$ written in unary,
  decide if there exists a covering execution with $n$ agents.
  \label{def:pb:cover}
\end{definition}
% \begin{definition}[Coverage]~
%   \begin{itemize}
%   \item[] Input: a topological graph $G$ and $n\in \mathds{N}$ written in unary;
%   \item[] Output: does there exists a covering execution with $n$ agents?
%   \end{itemize}
%   \label{def:pb:cover}
% \end{definition}

\begin{definition}[bReachability]
  Given a topological graph $G$, configuration~$c$ and
    $\ell\in \mathds{N}$ written in unary,
  decide if there is an execution $\langle c^1, \dots,
    c^{\ell'}\rangle$ in $G$ s.t. $\ell' \leq \ell$ and $c^{\ell'} = c$.
  \label{def:pb:breach}
\end{definition}
% \begin{definition}[bReachability]~

%   \begin{itemize}
%   \item[] Input: a topological graph $G$, $c$ a configuration and
%     $\ell\in \mathds{N}$ written in unary;
%   \item[] Output: does there exists an execution $\langle c^1, \dots,
%     c^{\ell'}\rangle$ in $G$ such that $\ell' \leq \ell$ and $c^{\ell'} = c$?
%   \end{itemize}
%   \label{def:pb:breach}
% \end{definition}

\begin{definition}[bCoverage]
  Given a topological graph $G$, $n,\ell\in\mathds{N}$ written in unary,
  decide if there exists a covering execution of length $\ell'$ such
    $\ell'\leq \ell$.
  \label{def:pb:bcover}
\end{definition}
% \begin{definition}[bCoverage]~

%   \begin{itemize}
%   \item[] Input: a topological graph $G$, $n,\ell\in\mathds{N}$ written in unary;
%   \item[] Output: does there exists a covering execution of length $\ell'$ such
%     $\ell'\leq \ell$?
%   \end{itemize}
%   \label{def:pb:bcover}
% \end{definition}

%For the rest of the paper,
We study the restrictions of the above problems to
classes of topological graphs. We denote by \pb{B}{c}, with $B$ one of the four
above problems restricted to the class of topological graph $c$ ($c$ can either
be $dir$ for directed, $nc$ for neighbor-communicable, $sm$ for sight-moveable
or $cc$ for complete-communication topological graphs).

\subsection{Known results}
\label{preli:known}

The complexity of the decision problem associated to the minimization of the
makespan is known to be NP-hard since \cite{Ratner:86}. Throughout the study of
MAPP, NP-hardness was shown to hold on planar graphs \cite{Yu:2015}
and, later, on 2D grid graphs \cite{Banfi:2017}. Variants
of MAPP have been studied such as the package-exchange
robot-routing problem \cite{Ma:2016} where the robots are anonymous but not
the package they exchange, is shown to be NP-hard. A class
of grid graphs was shown to be solvable in polynomial time
\cite{Wang:2009}.

The connected version of MAPP was introduced in \cite{Hollinger:2012}, in
which a topological graph discretizes the space and it is proved that the
existence of a plan for the reachability of a configuration of agents in a
bounded amount of steps is NP-hard:

\begin{theorem} \label{th:dir:lb:breach}
  \pbBReach{} restricted to undirected topological graphs is NP-hard
  \cite{Hollinger:2012}.
\end{theorem}

In \cite{dblp:conf/aaai/tateobrab18}, it is shown that deciding the existence of
a feasible plan is PSPACE-complete:

\begin{theorem} \label{th:dir:lb:reach}
    \pbReach{} restricted to undirected topological graphs is PSPACE-complete
    \cite{dblp:conf/aaai/tateobrab18}.
\end{theorem}

Authors prove this result for graphs with self-loops and a
base~\cite{dblp:conf/aaai/tateobrab18} as in our setting (see Discussion
following Theorem 1). The only difference with our setting is that the agents
start at a specific configuration in~\cite{dblp:conf/aaai/tateobrab18}.
Nevertheless, it can be shown easily that our problem is equivalent by
duplicating the base and adding edges so that the agents reach the initial
configuration at the second step.

%%% Local Variables:
%%% mode: latex
%%% TeX-master: "main"
%%% End:

In the rest of the paper, we study the upper bounds and the lower bounds
complexity of the defined decision problems on the previously defined
topological graphs. The following sections present our results, respectively,
for the general case, the neighbor-communicable graphs, sight-moveable graphs,
and complete-communication graphs.
%In Section \ref{sec:dir}, we show the complexity in the
%general case. Then, we study the neighbor-communicable restriction in Section
%\ref{sec:nc}. In Section \ref{sec:sm}, we restrict our study to sight-moveable.
%Finally, in Section \ref{sec:compl}, we study the complete-communication
%topological graphs.

\section{Directed Topological Graphs}
\label{sec:dir}

%\subsection{Upper bounds}

We start with the following upper bound which is obtained
by a straightforward guess and check algorithm:
\begin{proposition} \label{prop:dir:ub:bcover-breach}
	\pbBCoverageDir and \pbBReachDir are in NP.
\end{proposition}

In both cases, we can guess and check a path of bounded length in NP since the
input is encoded in unary.

We furthermore establish the following results:
%In this subsection, we show that \pbReachDir and \pbCoverageDir are
%PSPACE-complete and \pbBCoverageDir is in NP.

\begin{theorem}
  \pbCoverageDir and \pbReachDir are PSPACE-complete.
  \label{prop:dir:ub:cover-reach}
  \label{th:dir:compl:reach}
\end{theorem}

The upper bounds are obtained by a straightforward NPSPACE algorithm
that guesses an execution by keeping in memory the last configuration,
and, for \pbCoverageDir, the set of visited regions. We conclude with
Savitch's Theorem (NPSPACE=PSPACE)\cite{Savitch:70}.

The lower bound on \pbReachDir was proven in Theorem~\ref{th:dir:lb:reach}. We
now concentrate on \pbCoverageDir.

\begin{lemma}
  \label{lemma:dir:lb:cover}
  \pbCoverageDir is PSPACE-hard.
\end{lemma}
\begin{proof}
  The proof is by reduction from \pbReachDir in which the base node has a
  self-loop. As noted in the remark following Theorem~\ref{th:dir:lb:reach},
  this problem remains PSPACE-hard.
  We map an instance $(G, c)$ of \pbReachDir
  to the instance $G'$ of \pbCoverageDir where $G'$ is depicted in
  Fig.~\ref{figure:topologicgraphproblemcoverage}. Let~$k$ denote the number
  of agents in the instance~$(G,c)$. $G'$ contains $G$ as a subgraph, plus fresh
  nodes $v_1, \dots, v_k$ and $s_1, \dots, s_k$. An agent can move from any node
  of $G$ to $v_1$ and back.

    \begin{figure}[t]
        \begin{center}
            \newcommand{\ymax}{5.3}
            \newcommand{\ybase}{2.8}
            \scalebox{0.8}
            {
                \begin{tikzpicture}[scale=0.9]
                    \draw[copyofG] (-0.2,\ybase-0.3) rectangle (8.2,3.5);
                    \draw (1.5,\ybase) node (nb) {B};
                    \draw (1,4) node (s1) {$s_1$};
                    \draw (3,4) node (s2) {$s_2$};
                    \draw (7,4) node (sk) {$s_k$};
                    \draw (1,3.2) node (t1) {$c_1$};
                    \draw (3,3.2) node (t2) {$c_2$};
                    \draw (7,3.2) node (tk) {$c_k$};
                    \draw (1,4.8) node (v1) {$v_1$};
                    \draw (3,4.8) node (v2) {$v_2$};
                    \draw (7,4.8) node (vk) {$v_k$};
                    \draw (8,5.3) node (all) {all};
                    \draw[-stealth] (t1) -- (s1);
                    \draw[-stealth] (t2) -- (s2);
                    \draw[-stealth] (tk) -- (sk);
                    \draw[tobehere,communication] (s1) -- (s2);
                    \draw[tobehere,communication] (s2) -- (4.5,4);
                    \draw[dotted,color=blue!50!white] (4.5,4) -- (5.5,4);
                    \draw[tobehere,communication] (sk) -- (5.5,4);
                    \draw[tobehere,communication] (t1) -- (t2);
                    \draw[tobehere,communication] (t2) -- (4.5,3.2);
                    \draw[dotted,color=blue!50!white] (4.5,3.2) -- (5.5,3.2);
                    \draw[tobehere,communication] (tk) -- (5.5,3.2);
                    \path[tobehere,communication] (s1) edge[bend right=70] (nb);
                    %\addCommunicationB{s1}
                    \draw[-stealth] (s1) -- (v1);
                    \draw[-stealth] (s2) -- (v2);
                    \draw[-stealth] (sk) -- (vk);
                    \draw[stealth-stealth] (v1) to[bend right] (0.1,3.4);
                    \path (vk) edge [loop right] (vk);
                    \draw[-stealth] (nb) to [loop right] (nb);
                    \draw [-stealth, rounded corners = 8pt] (vk) -- (7,\ymax) -- (-1,\ymax) -- (-1, \ybase)  -- (nb);
                    \draw (v1) -- (1,\ymax);
                    \draw (v2) -- (3,\ymax);
                    \draw[tobehere,communication] (vk) -- (all);
                    \node at (4, 2.7) {copy of $G$};
                \end{tikzpicture}
            }
            \caption{Topological graph $G'$ constructed from the \pbReachDir-instance.}
            \label{figure:topologicgraphproblemcoverage}
        \end{center}
    \end{figure}
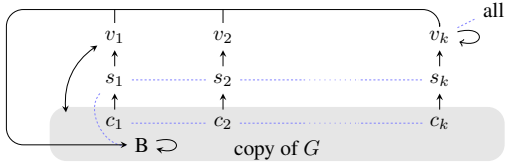
    Node $s_1$ can communicate with the base $B$, and node $v_k$ can communicate
    with all nodes of $G'$. Furthermore, we have the communication edges
    $(s_i,s_{i+1})$ and~$(v_i,v_{i+1})$ for all~$1\leq i\leq k-1$.
    Now we prove that the $k$ agents can progress to the
    configuration ($c_1,\dots,c_k$) in $G$  if and only if there exists a
    covering execution in~$G'$.

    ($\Rightarrow$) If the agents are in the configuration ($c_1,\dots,c_k$)
    then they can progress in one step to configuration ($s_1,\dots,s_k$). Then,
    they have no choice but progress to the configuration ($v_1,\dots,v_k$).
    Once in this configuration, the agent placed on the node $v_k$
    communicates with the base and with all other agents. %with any agent, placed on any node, and to the base $B$.
    This agent stays at $v_k$. Meanwhile the agent placed on the
    node $v_1$ will visit all unvisited nodes of $G$ and come back to $v_1$
    while keeping communication to the base through the agent placed on $v_k$.
    Meanwhile, agents placed on $v_2,\dots,v_{k-1}$ come back to $B$. Finally,
    when all the nodes have been visited, both agents on $v_1$ and $v_k$ come
    back to $B$.

    \newcommand{\timesk}{t_{s_k}}

    ($\Leftarrow$) If there exists a covering execution of the whole graph $G'$,
    it means all nodes have been visited. In particular, node $s_k$ has been
    visited and let us consider the first time $\timesk$ when $s_k$ is visited.
    Time $\timesk-1$ denotes the time just before $\timesk$.

    \begin{fact} \label{fact:nodesoutsideGprimeunvisited}
      At time $\timesk-1$, no node $v_i$ and no node~$s_i$ were visited.
    \end{fact}

    \begin{proof}
        Suppose by contradiction that a node $v_i$ was visited by some agent
        before $\timesk$, then the only possibility such an agent to communicate
        to the base is that there is also an agent at $v_k$ at time $\timesk$.
        But then, it means that $s_k$ was visited strictly before $\timesk$,
        leading to a contradiction. Thus, no node $v_i$ were visited at time
        $\timesk$ (thus at time $\timesk-1$).

        As no node $v_i$ are visited before $\timesk$, no node $s_i$ are visited
        before $\timesk-1$.
    \end{proof}

    \begin{fact} \label{fact:reachc1ck}
        At time $\timesk-1$, the configuration is $\langle c_1, \dots, c_k\rangle$.
    \end{fact}

    \begin{proof}
        At time $\timesk$, as the agent at $s_k$ needs to communicate with the base,
        the only
        possibility is that the configuration is $\langle s_1, \dots, s_k\rangle$. Thus, the
        only possibility is that configuration is $\langle c_1, \dots, c_k\rangle$.
    \end{proof}

    Facts~\ref{fact:nodesoutsideGprimeunvisited} implies that the prefix from
    time 0 to time $\timesk-1$ of the covering execution is an execution in $G$.
    Fact~\ref{fact:reachc1ck} implies that sub-execution reaches $\langle
    c_1,\dots,c_k\rangle$.
\end{proof}

% By Proposition~\ref{prop:dir:ub:cover-reach} and Theorem~\ref{th:dir:lb:cover}:

% \begin{theorem} \label{th:dir:compl:cover}
%     \pbCoverageDir is PSPACE-complete.
% \end{theorem}
% \begin{theorem}
%     \label{th:dir:compl:cover}
%   \pbCoverageDir is PSPACE-complete.
% \end{theorem}

%%% Local Variables:
%%% mode: latex
%%% TeX-master: "main"
%%% End:

\section{Neighbor-Communicable Topological Graphs}
\label{sec:nc}

In this subsection, we show that our problems remain hard
for neighbor-communicable graphs.

%The \pbReach{} on the neighbor-communicable topological graphs was proven
%to be PSPACE-complete in \cite{dblp:conf/aaai/tateobrab18}. 
%\todo{Ocan: Removed ref to Tateo}
%We study the \pbCoverage{}problem for this subclass.

\begin{theorem} 
  \label{thm:nc:lb:cover}
  \pbCoverageNC is PSPACE-complete.
\end{theorem}

\begin{proof}
  The upper bound is given by Theorem~\ref{prop:dir:ub:cover-reach}.

  For the lower bound on \pbCoverageNC, the reduction given in
  Figure~\ref{figure:topologicgraphproblemcoverage} is not adapted 
  for neighbor communicable graphs.
  Indeed, all nodes may be visited  although $c_1, \dots, c_k$
  is not reached: $v_1$ and $v_k$ can be reached by two lines of agents
  connected to the base, making the coverage of the full graph possible.
  We nevertheless give a similar reduction by adapting the previous reduction.
  \iffull

  \begin{figure}[t]
        \begin{center}
            \newcommand{\xretourgauche}{-1}
            \tikzstyle{circlenode} =  [circle, draw=black, fill=black, inner sep=0.8mm]
            \newcommand{\ligne}[1]{
            \node[circlenode]  (s1) at (0,#1)  {};
            \node[circlenode]  (s1) at (0.3,#1)  {};
            \node[circlenode]  (s1) at (1,#1)  {};
            \node[circlenode] (s2) at (3,#1)   {};
            \node[circlenode]  (sk) at (7,#1)  {};
            \draw[-stealth] (1, #1-0.5) -- (s1);
            \draw[-stealth] (3, #1-0.5) -- (s2);
            \draw[-stealth] (7, #1-0.5) -- (sk);
            \draw[tobehere,communication] (s1) -- (s2);
            \draw[tobehere,communication] (s2) -- (4.5,#1);
            \draw[dotted,color=green] (4.5,#1) -- (5.5,#1);
            \draw[tobehere,communication] (sk) -- (5.5,#1);}

            \newcommand{\lignevdots}[1]{
            \node[]  (s1) at (1,#1)  {$\vdots$};
            \node[] (s2) at (3,#1)   {$\vdots$};
            \node[]  (sk) at (7,#1)  {$\vdots$};
            \draw[-stealth] (1, #1-0.7) -- (s1);
            \draw[-stealth] (3, #1-0.7) -- (s2);
            \draw[-stealth] (7, #1-0.7) -- (sk);
            }

            \newcommand{\vnodeY}{8.7}
            \newcommand{\posBaseY}{2.75}
            \scalebox{1}
            {
                \begin{tikzpicture}[xscale=0.75,yscale=0.9]
                \draw[copyofG] (-0.2,\posBaseY-0.3) rectangle (8.2,4);
%                \draw (4,\posBaseY	) node[draw] (nb) {B};
                \draw (4,\posBaseY) node (nb) {B};
                \draw (1,3.75) node (t1) {$c_1$};
                \draw (3,3.75) node (t2) {$c_2$};
                \draw (7,3.75) node (tk) {$c_k$};

                \draw[tobehere,communication] (t1) -- (t2);
                \draw[tobehere,communication] (t2) -- (4.5,3.75);
                \draw[dotted,color=green] (4.5,3.75) -- (5.5,3.75);
                \draw[tobehere,communication] (tk) -- (5.5,3.75);

                \ligne {4.5}
                \addCommunicationB {s1}
                \lignevdots {5.25}
                \ligne {5.9}
                \addCommunicationB {s1}

                \node at (7.5, 5.8) {$s_k$};

                \draw[decoration={brace,mirror,raise=5pt},decorate]
                (8,4.3) -- node[right=10pt] {$k+1$} (8, 6);

                \ligne {6.5}
                \addCommunicationB {sk}
                %			\ligne {10.5}
                %			\addCommunicationB {sk}
                \lignevdots {7.25}
                \ligne {7.9}
                \addCommunicationB {sk}

                \draw[decoration={brace,mirror,raise=5pt},decorate]
                (8,6.3) -- node[right=10pt] {$k+1$} (8,8);

                \draw (1,\vnodeY) node (v1) {$v_1$};
                \draw (3,\vnodeY) node (v2) {$v_2$};
                \draw (7,\vnodeY) node (vk) {$v_k$};
                \draw (8,\vnodeY+0.5) node (all) {all};
                \draw[-stealth] (s1) -- (v1);
                \draw[-stealth] (s2) -- (v2);
                \draw[-stealth] (sk) -- (vk);
                \draw[-stealth, rounded corners = 8pt] (v1) -- (0.3, \vnodeY-0.3) -- (0.3,4);
                \draw[stealth-, rounded corners = 8pt] (v1) -- (0, \vnodeY) -- (0,4);

                \path (vk) edge [loop right] (vk);
                \draw[-stealth] (nb) to [loop right] (nb);
                \draw [-stealth, rounded corners = 8pt] (vk) -- (7,\vnodeY+0.5) -- (\xretourgauche,\vnodeY+0.5) -- (\xretourgauche, \posBaseY)  -- (nb);
                \draw (v1) -- (1,\vnodeY+0.5);
                \draw (v2) -- (3,\vnodeY+0.5);

                \draw[tobehere,communication] (vk) -- (all);
                \node at (4, 3.25) {copy of $G$};
                \end{tikzpicture}
            }
            \caption{Topological graph of the \pbCoverageNC-instance constructed from the \pbReachNC-instance for the case of neighbor-communicable topological graph.}
            \label{figure:topologicgraphproblemcoveragefixedforneighborcommunicable}
        \end{center}
    \end{figure}
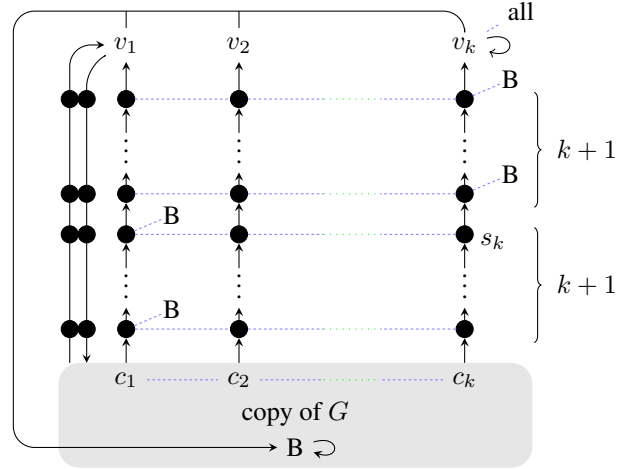

	The corrected construction is given in Figure
	\ref{figure:topologicgraphproblemcoveragefixedforneighborcommunicable}. When
	configuration $\langle c_1,...,c_k\rangle$ is reached, the agents go through
	a first layer of length $k+1$ in which the first agent can communicate with
	$B$. Then they go through another layer of length $k+1$ in which the
	$k^{th}$ agent can communicate with $B$. This way, it is mandatory that all
	agents move at the same time to visit $\langle v_1,\dots,v_k\rangle$. Once
	the $k^{th}$ agent is at $v_k$, all agents can communicate with $B$ wherever
	they are, so they can visit remaining states in the copy of $G$. Now let us
	prove that $\langle c_1, \dots, c_k\rangle$ is reachable in $G$ iff it is
	possible to cover all nodes in $G'$.

	($\Rightarrow$) If $\langle c_1, \dots, c_k\rangle$ is reachable in $G$,
	then we extend the execution to reach $\langle v_1, \dots, v_k\rangle)$ and
	by the same trick as in Figure \ref{figure:topologicgraphproblemcoverage},
	the agent that reaches $v_1$ visits all the remaining unvisited nodes in
	$G$. Thus, we extend the execution for covering all nodes in $G'$.
	
	$(\Leftarrow)$ Suppose all nodes are visited in $G'$. In particular, $v_1$
	and $v_k$ are visited. Let us consider the first moment $t_{v_i}$ when a
	node $v_i$ is visited.
	
	\begin{fact} \label{fact:nodesv1vk}
		At the first moment, the configuration of the agents is $\langle v_1,
		\dots, v_k\rangle$.
	\end{fact}
	
	\begin{proof}
		Let us prove that there is an agent at $v_k$. Suppose that at that
		moment there is no agent at $v_k$. Due to the topological graph $G'$, the
		agent at $v_i$ is disconnected from the base since  nodes that
		communicate directly to $B$ are too far from $v_i$: indeed, the top
		$k+1$-grid is too long and, for $i=1$, the path on left between $v_1$
		and the copy of $G$ is too long. Contradiction.
		
		The agent at $v_k$ came from the unique $2k+2$-long path from $c_k$ to
		$v_k$. Actually, $k+1$ steps before - let us call this moment $t_{s_k}$,
		she was on $s_k$. But at that time, due to the topological graph, there
		are $k$ agents on the row containing $s_k$, otherwise the agent at $s_k$
		would have been disconnected from the base (the bottom $k+1$-grid is too
		long).
		
		So $k+1$ times later $t_{s_k}$, all the $k$ agents are at $\langle v_1,
		\dots, v_k\rangle$.
	\end{proof}
	
	Taking Fact \ref{fact:nodesv1vk} as granted, we consider time $t$ that is
	$2k+2$ steps before and we clearly have the following fact.
	
	\begin{fact}
		At time $t$, the configuration is $\langle c_1, \dots, c_k\rangle$.
	\end{fact}

	Moreover, the following fact holds.
	
	\begin{fact}
		At time $t$, no node outside $G$ were visited.
	\end{fact}
	
	\begin{proof}
		By contradiction, if some node outside $G$ were visited, it means that
		some agent went out the copy of $G$. By definition of $G'$, it would
		mean that a node $v_i$ would have been visited, before time $t$, hence
		strictly before $t_{v_i}$. Contradiction.
	\end{proof}

	To sum up, the prefix of the execution from $\langle B, \dots, B\rangle$ to
	$\langle c_1, \dots, c_k\rangle$ is fully inside the copy of $G$. So
  $\langle c_1, \dots, c_k\rangle$ is reachable in $G$.
  
  \else

  The details are given in the long version.

  \fi

\end{proof}

%By Proposition~\ref{prop:nc:lb:cover} and

% \begin{theorem}
% \label{th:nc:compl:cover}
% \pbCoverageNC is PSPACE-complete.
% \end{theorem}

%%% Local Variables:
%%% mode: latex
%%% TeX-master: "main"
%%% End:

\section{Sight-Moveable Topological Graphs}
\label{sec:sm}

%\textbf{\textcolor{red}{Proofs density 5.1-5.4}}

In this subsection, we show that \pbReachSM and \pbCoverageSM
are in LOGSPACE while the bounded version \pbBReachSM is
NP-complete.

\subsection{Upper bounds}

The results in this subsection are based on a result of Reingold~\cite{Reingold:2008}, who
proved that the problem of checking the connectivity of two nodes $s$ and $t$ in an
undirected graph, namely \USTCONN, is in LOGSPACE.

\begin{theorem}
	\label{theorem:USTCONNinLOGSPACE}
	\USTCONN~ is in LOGSPACE \cite{Reingold:2008}.
\end{theorem}

\begin{proposition}\label{prop:sm:ub:reach}
	\pbReachSM is in LOGSPACE.
\end{proposition}

\begin{proof}
  The idea of the proof is to reduce \pbReachSM to \UCONN, that is the problem
  of deciding whether an undirected graph is connected. From
  Theorem~\ref{theorem:USTCONNinLOGSPACE}, we can reduce \UCONN~ to \USTCONN~
  by simply looping over all pairs of nodes ($\sourcenode, \targetnode)$ and
  checking for a path from $\sourcenode$ to $\targetnode$. Therefore, \UCONN~
  is in LOGSPACE.

  Now we describe the logarithmic space reduction of \pbReachSM
  \linebreak[4]
  to \UCONN. Let $G=\langle \setnodes, \moves, \coms \rangle$ a sight-moveable
    topological graph and $c$ a configuration. Let $\setnodes' = \set{c_1,
    \dots, c_n, \basenode}$. The configuration $c$ is reachable iff the
    restriction of $\myundirectedgraph := (\setnodes, \coms)$ to the nodes in
    $\setnodes'$ is $\coms$-connected. Indeed, if it is, then $c$ is reachable:
    each agent follows some $\moves$-path from $\basenode$ to $c_i$ contained in
    a $\coms$-path from $\basenode$ to $c_i$. In other words, $(G, c)$ is a
    positive \pbReachSM-instance iff $\myundirectedgraph$ is a positive
    \UCONN-instance. The reduction is in logarithmic space: we compute
    $\myundirectedgraph$ by enumerating all $(u, v)$ $\coms$-edges in $G$, and
    we output $(u, v)$ when $u, v \in \setnodes'$. We recall that we only take
    into account the working memory for computing $\myundirectedgraph$; the
    output -- $\myundirectedgraph$ itself -- is not taken into account in the
    used space (see e.g. \cite{DBLP:books/daglib/0086373}, Ch. 8, Def. 8.21).
\end{proof}

%We denote $Comm(S)=\{v \in V / \exists u \in S, v \coms u\}$ the set of nodes
%which communicates with a node in $S$.
%We obtain that the algorithm \pbReachSM is sound and complete for the problem
%of reachability of a configuration in a topological graph. The \pbReachSM
%problem is then in P given that the algorithm \pbReachSM computes the fixed
%point on $S$ in polynomial-time.

% \end{comment}

\begin{proposition}\label{prop:sm:ub:cover}
  \pbCoverageSM is in LOGSPACE.
\end{proposition}

\begin{proof}
  First we prove that the bounded version of the connectivity in undirected
  graphs is also in LOGSPACE.

  \begin{lemma}
    \label{lemma:boundedUSTCONNinLOGSPACE}
    Bounded-\USTCONN, that is the problem, giving an undirected graph
    $\myundirectedgraphbounded$, two nodes $\sourcenode, \targetnode$, an
    integer $n$ written in binary, of deciding whether there is a path of
    length at most $n$ from $\sourcenode$ to $\targetnode$ in $G$ is in
    LOGSPACE.
  \end{lemma}

  \begin{proof} We reduce Bounded-\USTCONN~ to \USTCONN~ in logarithmic space as
    follows. From a Bounded-\USTCONN~ instance $(\myundirectedgraphbounded,
    \sourcenode, \targetnode, n)$ we construct in logarithmic space a
    \USTCONN~ instance $(\myundirectedgraph, \sourcenode', \targetnode')$:
    \begin{inparaenum}
    \item The nodes of $\myundirectedgraph$ are pairs $(v, j)$ where $v$
      is a node of $\myundirectedgraphbounded$ and $j$ is an integer in
      $\set{0,n}$ but smaller than the number of nodes in
      $\myundirectedgraph$;
    \item $\myundirectedgraph$ contains an edge between $(v, j)$ and
      $(v', j+1)$ when there is an edge between $v$ and $v'$ in
      $\myundirectedgraphbounded$ or when $v = v'$;
    \item $\sourcenode' = (\sourcenode, 0)$ and $\targetnode' =
      (\targetnode', n)$.
    \end{inparaenum}
  \end{proof}

  Now let $G=\langle V, \moves, \coms \rangle$ be a sight-moveable topologic
  graph and $n$ an integer written in binary. To test whether $(G, n)$ is a
  positive instance of Coverage, it suffices to check that there is a path
  from any node $v$ to the base $B$ with at most $n$ communication edges. To
  do that, we test sequentially, for all $v$, that $((V, \coms), v, B, n)$ is
  a positive instance of Bounded-\USTCONN~. Thus, we obtain an algorithm in
  logarithmic space to decide Coverage.

\end{proof}

%%% Local Variables:
%%% mode: latex
%%% TeX-master: "main"
%%% End:

\subsection{Lower bounds}

We now focus on the NP lower bound of \pbBReachSM. 
% We
% recall the 3-satisfiability (3-SAT) problem which is
% NP-complete~\cite{DBLP:conf/coco/Karp72}.

% \begin{definition}[3-SAT] The 3-satisfiability problem is the following decision
%   problem:
%   \begin{itemize}
%   \item[] Input: a set of $n$ variables $\booleanvariablegeneric_1, \dots,
%     \booleanvariablegeneric_n$ and a set of $m$ disjunctive clauses
%     $\clausegeneric_1, \dots, \clausegeneric_m$. Each clause contains up to 3
%     literals (a variable or its negation).
% %  \item[] Output: does there exists an assignment for the variables satisfying
% %  all clauses ?
%   \item[] Output: is there a variable assignment satisfying all clauses?
% %does there exists an assignment for the variables satisfying
% %  all clauses ?
%   \end{itemize}
% \end{definition}

% \begin{theorem}
%   3-SAT problem is NP-complete \cite{DBLP:conf/coco/Karp72}.
% \end{theorem}

\iffull
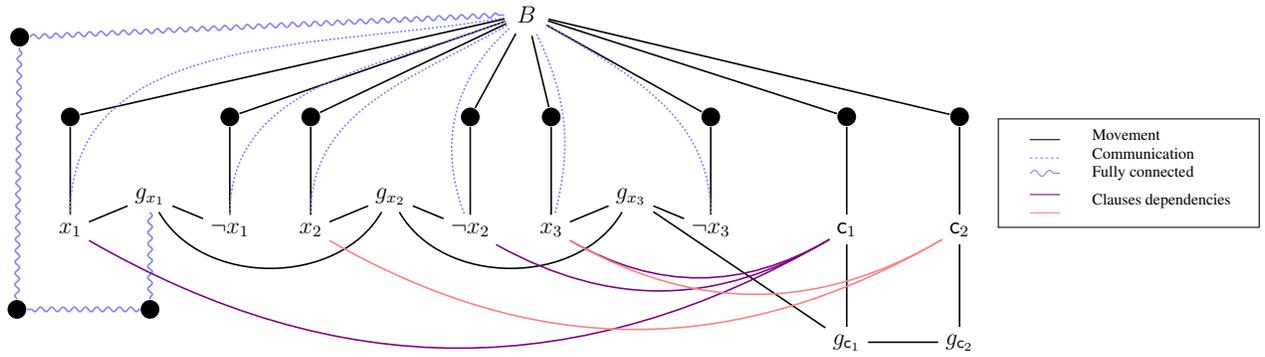
\begin{figure*}
    \begin{center}
    \vspace{1cm}
    \scalebox{0.75}{
    \begin{tikzpicture}[node distance=2cm,thick,
      main node/.style={font=\sffamily\Large\bfseries,minimum size=0.5cm},
      staging node/.style={circle,fill=black}]
  
      \node[main node,circle] (B) {$\basenode$};
      \draw[use as bounding box] (B.center)+(4cm,0);
      \node (VAR) [below left=0cm and 2cm of B] {};
      \node (VAR1) [left=4cm of VAR] {};
      \node (VAR3) [right=4cm of VAR] {};
  
      \node[staging node] (SX1) [below left of=VAR1] {};
      \node[main node] (X1) [below of=SX1] {$\booleanvariablegeneric_1$};
      \node[staging node] (SNX1) [below right of=VAR1] {};
      \node[main node] (NX1) [below of=SNX1] {$\neg \booleanvariablegeneric_1$};
      \node[main node] (GX1) [below right of=SX1] {$g_{\booleanvariablegeneric_1}$};
  
      \node[staging node] (SX2) [below left of=VAR] {};
      \node[main node] (X2) [below of=SX2] {$\booleanvariablegeneric_2$};
      \node[staging node] (SNX2) [below right of=VAR] {};
      \node[main node] (NX2) [below of=SNX2] {$\neg \booleanvariablegeneric_2$};
      \node[main node] (GX2) [below right of=SX2] {$g_{\booleanvariablegeneric_2}$};
  
      \node[staging node] (SX3) [below left of=VAR3] {};
      \node[main node] (X3) [below of=SX3] {$\booleanvariablegeneric_3$};
      \node[staging node] (SNX3) [below right of=VAR3] {};
      \node[main node] (NX3) [below of=SNX3] {$\neg \booleanvariablegeneric_3$};
      \node[main node] (GX3) [below right of=SX3] {$g_{\booleanvariablegeneric_3}$};
  
      \node[staging node] (P3) [below of=GX1] {};
      \node[staging node] (P2) [left=2cm of P3] {};
      \node[staging node] (P1) [left=2cm of VAR1] {};
  
      \node (CLA) [below right=0cm and 3cm of B] {};
      \node (CLA1) [right=2cm of CLA] {};
      \node (CLA2) [right=4cm of CLA] {};
      
      \node[staging node] (SC1) [below=1.1cm of CLA1] {};
      \node[main node] (C1) [below of=SC1] {$\clausegeneric_1$};
      \node[main node] (GC1) [below of=C1] {$g_{\clausegeneric_1}$};
  
      \node[staging node] (SC2) [below=1.1cm of CLA2] {};
      \node[main node] (C2) [below of=SC2] {$\clausegeneric_2$};
      \node[main node] (GC2) [below of=C2] {$g_{\clausegeneric_2}$};
      
      \foreach \i in {1,...,3}{\path (B) edge (SX\i);}
      \foreach \i in {1,...,3}{\path (B) edge (SNX\i);}
      \foreach \i in {1,...,3}{\path (SX\i) edge (X\i);}
      \foreach \i in {1,...,3}{\path (SNX\i) edge (NX\i);}
      \foreach \i in {1,...,3}{\path (X\i) edge (GX\i);}
      \foreach \i in {1,...,3}{\path (NX\i) edge (GX\i);}
  
      \foreach \i in {1,...,2}{\path (B) edge (SC\i);}
      \foreach \i in {1,...,2}{\path (SC\i) edge (C\i);}
      \foreach \i in {1,...,2}{\path (C\i) edge (GC\i);}
  
      \path[communication,tobehere]
      (B) edge [out=190, in=90] (X1)
      edge [out=200, in=90] (NX1)
      edge [out=210, in=90] (X2)
      edge [bend right=35] (NX2)
      edge [bend left=20] (X3)
      edge [out=330, in=90] (NX3)
      % edge [out=337, in=90] (C1)
      % edge [out=345, in=90] (C2)
      ;
      \path
      (GX1) edge [bend right=60] (GX2)
      (GX2) edge [bend right=60] (GX3)
      (GX3) edge (GC1)
      (GC1) edge (GC2);
      \draw[tobehere,decorate,decoration={snake,amplitude=.4mm,segment length=2mm}] (GX1) -- (P3);
      \draw[tobehere,decorate,decoration={snake,amplitude=.4mm,segment length=2mm}] (P3) -- (P2);
      \draw[tobehere,decorate,decoration={snake,amplitude=.4mm,segment length=2mm}] (P2) -- (P1);
      \draw[tobehere,decorate,decoration={snake,amplitude=.4mm,segment length=2mm}] (P1) -- (B);
      %(GC2) edge [bend right] (CLA2)
      %(CLA2) edge [bend right] (B)
      %(CLA2) edge [bend right] (CLA2.center)
      %(CLA2) edge [bend left] (CLA2.center)
      
      \path
  
      (C1) edge [bend left, blue!50!red](X1)
      edge [bend left, blue!50!red] (NX2)
      edge [bend left, blue!50!red] (X3)
  
      (C2) edge [bend left, red!50!white] (X2)
      edge [bend left, red!50!white] (X3)
      
      ;
    \end{tikzpicture}
    }
    
    % \vspace{4.2cm}
    % \vspace{5cm}
    \hspace{13.2cm}		
    \raisebox{-1.5cm}{
      \fbox{
        \tiny
        \begin{tabular}{ll}
          \tikz{\draw (0, 0) edge (0.4, 0);} & Movement \\
          \tikz{\draw[communication, tobehere] (0, 0) edge (0.4, 0);} & Communication \\
          \tikz{\draw[tobehere,decorate,decoration={snake,amplitude=.4mm,segment
          length=2mm}] (0, 0) -- (0.4, 0);} & Fully connected \\
          \tikz{\draw[blue!50!red] (0, 0) edge (0.4, 0);} & \multirow{2}{*}{Clauses dependencies}\\
          \tikz{\draw[red!50!white] (0, 0) edge (0.4, 0);} &
        \end{tabular}
        }
    }
    \end{center}
    \vspace{1.8cm}
    \caption{Graph construction from the formula $(x_1 \vee \neg x_2 \vee x_3) \wedge (x_2 \vee x_3)$.}
    \label{fig:cons}
  \end{figure*}

  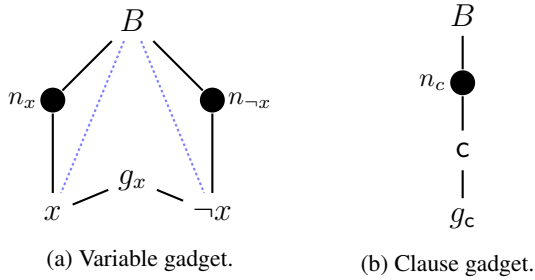
\begin{figure}
    \centering
    \begin{subfigure}{.23\textwidth}
      \centering
      \begin{tikzpicture}[node distance=1.5cm,
                      thick,
                      main node/.style={,font=\sffamily\Large\bfseries,minimum size=0.5cm},
                      staging node/.style={circle,fill=black}]
  
        \node[main node] (1) {$\basenode$};
        \node[staging node] (2) [below left of=1]{};
        \node[left of=2, node distance=0.4cm] {$n_x$};
        \node[main node] (3) [below of=2] {$\booleanvariablegeneric$};
        \node[staging node] (4) [below right of=1] {};
        \node[right of=4, node distance=0.5cm] {$n_{\lnot x}$};
        \node[main node] (5) [below of=4] {$\neg \booleanvariablegeneric$};
  
        \node[main node] (6) [below right of=2] {$g_\booleanvariablegeneric$};
  
        \path[every node/.style={font=\sffamily\small}]
        (1) edge (2)
        (2) edge (3)
        (3) edge (6)
  
        (1) edge (4)
        (4) edge (5)
        (5) edge (6)
        ;
  
        \path[communication,tobehere]
        (1) edge (3)
        (1) edge (5)
        ;
    \end{tikzpicture}
    \caption{Variable gadget.}
    \label{fig:Gvar}
    \end{subfigure}%
    \begin{subfigure}{.23\textwidth}
      \centering
      \begin{tikzpicture}[node distance=0.9cm,
                      thick,
                      main node/.style={font=\sffamily\Large\bfseries,minimum size=0.5cm},
                      staging node/.style={circle,fill=black}]
  
        \node[main node] (1) {$\basenode$};
        \node[staging node] (2) [below of=1] {};
        \node[left of=2,node distance=0.4cm] {$n_c$};
        \node[main node] (3) [below of=2] {$\clausegeneric$};
        \node[main node] (4) [below of=3] {$g_\clausegeneric$};
  
        \path[every node/.style={font=\sffamily\small}]
        (1) edge (2)
        (2) edge (3)
        (3) edge (4)
        ;
  
        % \path[dashed,tobehere]
        % (1) edge [bend right] (3)
        % ;
    \end{tikzpicture}
    \caption{Clause gadget.}
    \label{fig:Gcla}
    \end{subfigure}
    \caption{Translation gadgets}
  \end{figure}
\fi

\begin{proposition} \label{prop:sm:lb:breach}
  \pbBReachSM is NP-hard for a fixed execution length $\ell\geq 3$.
\end{proposition}
\begin{proof}
  The proof is by polynomial time reduction from 3-SAT problem (see \cite{DBLP:conf/coco/Karp72}).
  Given a 3-SAT instance, set of clauses $c_1,\ldots,c_m$ with variables~$x_1,\ldots,x_n$,
  we describe the
  construction of an instance $(G, c)$ of \pbBReachSM \linebreak with~$k=n+m$ agents.
%  \todo{Possible confusion between~$c$ and~$c_1,\ldots,c_m$}

% We construct the graph $G=\langle V,E,C\rangle$ from an instance of 3-SAT.
%The topological graph contains a node for each literal, a goal node for each variable, a node and a goal node for each clause. Then, a clause exactly communicates with the literal it is
  %associated with. In order to reach a clause, an agent must be placed at one of
 % the literals of the former to preserve communication.
  %
  The topological graph $G=\langle\setnodes, \moves, \coms \rangle$ is constructed
  as follows. We start by placing the base $\basenode$ from which the agents start
  their mission.

  Please recall that a sight-moveable graph is also a neighbor-communicable graph so
  all movements edges are also communication edges in the construction below
  even if not explicitly stated.

  For each variable $\booleanvariablegeneric$, we construct a gadget composed of
  5 nodes connected to the base depicted in Figure \ref{fig:Gvar}:  nodes
  $\booleanvariablegeneric$, $\neg \booleanvariablegeneric$, staging nodes
  $n_\booleanvariablegeneric$, $n_{\neg \booleanvariablegeneric}$ and a
  \emph{goal} node $g_\booleanvariablegeneric$. We add movement edges from $B$
  to $n_\booleanvariablegeneric$, from $n_\booleanvariablegeneric$ to
  $\booleanvariablegeneric$ and from $\booleanvariablegeneric$ to
  $g_\booleanvariablegeneric$ (resp. from $B$ to $n_{\neg
  \booleanvariablegeneric}$, from $n_{\neg \booleanvariablegeneric}$ to $\neg
  \booleanvariablegeneric$ and from $\neg \booleanvariablegeneric$ to
  $g_\booleanvariablegeneric$). As for the communication, the node
  $\booleanvariablegeneric$ (res. $\neg \booleanvariablegeneric$) communicates
  with the base. %The staging nodes are not labeled in the figures for clarity.

  For each clause $\clausegeneric$, we construct a gadget composed of 3
  nodes depicted in Figure \ref{fig:Gcla}. We create a node $\clausegeneric$, a
  staging node $n_\clausegeneric$ and a goal node $g_\clausegeneric$. We add
  movement edges from $B$ to $n_c$, from $n_\clausegeneric$ to $\clausegeneric$
  and from $\clausegeneric$ to $g_\clausegeneric$. 
  %As in the variable gadget,
  %the node $\clausegeneric$ communicates with the base. 
  The communication between a clause $\clausegeneric$ and a literal
  $\booleanvariablegeneric$ or $\neg \booleanvariablegeneric$ is dictated by the
  existence of the literal in the clause: $\clausegeneric_i\coms
  \booleanvariablegeneric_j$ if and only if~$\booleanvariablegeneric_j \in
  \clausegeneric_i$; and $\clausegeneric_i\coms \lnot \booleanvariablegeneric_j$
  if and only if~$\lnot \booleanvariablegeneric_j \in \clausegeneric_i$.

  We add movement edges from
  $g_{\booleanvariablegeneric_i}$ to $g_{\booleanvariablegeneric_{i+1}}$,
  and from $g_{\clausegeneric_i}$ to~$g_{\clausegeneric_{i+1}}$ for all~$1 \leq i < n$,
  as well as we from
   $g_{\booleanvariablegeneric_n}$ to $g_{\clausegeneric_1}$.
   Last, we add a fully connected path containing 3 fresh nodes from $g_{\booleanvariablegeneric_1}$ to the base such that $g_{\booleanvariablegeneric_1} \coms \basenode$,
   in the sense that all nodes of this path have communication edges between them.
  This translation is polynomial in the number of clauses and variables.
  The construction is depicted in Figure~\ref{fig:cons}.
  The snake-like path from $g_{\booleanvariablegeneric_1}$ to $\basenode$ is the
  fully connected path.

  From a 3-SAT instance, one can construct the graph $G$ and ask for an
  execution of length 3 to reach the configuration \linebreak $\langle
  g_{\booleanvariablegeneric_1}, \dots, g_{\booleanvariablegeneric_n},
  g_{\clausegeneric_1}, \dots, g_{\clausegeneric_m}\rangle$.

  \iffull
  \begin{fact}
    $G$ is a sight-moveable topological graph.
  \end{fact}
  \begin{proof}
    One can see that the single communication edges created by the construction,
    apart from the ones induced by the movement, are the communication between
    the base $B$ and the nodes $\booleanvariablegeneric_i$ and $\neg
    \booleanvariablegeneric_i$. Hence, a path does exist under the communication
    of $B$ to reach $\booleanvariablegeneric_i$.
  \end{proof}

  Now let us prove that a 3-SAT instance is satisfiable iff there exists an
  execution of at most 3 steps in the graph $G$.

  ($\Rightarrow$) We show that if a 3-SAT instance is satisfiable then there
  exists an execution of at most 3 steps in the graph $G$ built from it. Let
  $val$ be a truth assignment which satisfies the instance.
  Recall that there are~$n+m$ agents.
  The first step of the execution consists in moving an agent in each~$n_{c_i}$,
  and for each variable~$x_j$, moving one agent to~$n_{x_j}$ if the $val(x_j)=1$
  and to~$n_{\lnot x_j}$ otherwise.
  Note that all staging nodes communicate with~$\basenode$ since the graph
  is neighbor-communicable.

  In the second step, all agents progress to their unique successors other than~$\basenode$.
  While all nodes~$x_j$ and~$\lnot x_j$ are connected to~$\basenode$,
  a node~$c_i$ is connected to~$\basenode$ if and only if there is an
  agent in one of its literals. This is the case since $val$ satisfies the formula.
  In the third step of the execution, agents go to states~$g_{x_j}$ and~$g_{c_i}$.
  Here, the connection with the base is ensured since~$g_{x_1}$ is connected
  to it, and~$g_{x_2}$ is connected to~$g_{x_1}$, $g_{x_3}$ is connected to~$g_{x_2}$ and
  so on.

  This execution is thus a solution of \pbBReachSM with bound $\ell=3$.

  ($\Leftarrow$) We now show that if there exists an execution of at most 3
  steps in the graph $G$ constructed from a 3-SAT instance, then the instance is
  satisfiable. Assume we have an execution $e$ of at most 3 steps with the last
  configuration being $\langle g_{\booleanvariablegeneric_1}, \dots,
  g_{\booleanvariablegeneric_n}, g_{\clausegeneric_1}, \dots,
  g_{\clausegeneric_m}\rangle$.

  The only shortest path from $\basenode$ to~$g_{c_i}$ is of length~$3$ and goes through
  $n_{c_i}$. For states $g_{x_j}$, the only shortest paths are also of length~$3$
  and go through either $n_{x_j}$ or~$n_{\lnot x_j}$. Thus, in order to reach the given
  target configuration, at the initial step, agents must cover the states~$n_{c_i}$
  and either~$n_{x_j}$ or~$n_{\lnot x_j}$ for all $i,j$. At the second step,
  following the above mentioned shortest paths, agents will be at states~$c_i$
  and either~$x_j$ or~$\lnot x_j$ depending on the staging nodes they were occupying.
  The last step is the target configuration.
  Since the agents are connected at the second, it follows that for each clause~$c_i$,
  the state corresponding to some literal of~$c_i$ is occupied by an agent.
  Thus the valuation on variables encoded by the choices of the agents satisfies the
  3-SAT instance.

  \else
  The rest of the proof is given in the long version.
  \fi
  % First, given the length of the execution, one
  % can observe that for each variable $\booleanvariablegeneric_i$ only the node
  % $\booleanvariablegeneric_i$ or $\neg \booleanvariablegeneric_i$ has been
  % visited, with $0<i\leq n$. Second, the $m$ agents in the clause nodes
  % $g_\clausegeneric$ required an agent to be located at one of the variable
  % nodes connected to it. Given that a clause is only connected to its literals,
  % an agent can access a clause iff another is placed at a node which satisfies
  % it.
\end{proof}

From Propositions~\ref{prop:dir:ub:bcover-breach} and~\ref{prop:sm:lb:breach}, we have:

\begin{theorem} \label{th:sm:compl:breach}
  \pbBReachSM is NP-complete.
\end{theorem}

%%% Local Variables:
%%% mode: latex
%%% TeX-master: "main"
%%% End:

\section{Complete-Communication Topological Graphs}
\label{sec:compl}

%In this section, we show that \pbBReachCC is in LOGSPACE and that \pbBCoverageCC
%is NP-complete.

%\subsection{Upper bounds}

The following result relies on the fact that the communication constraints are
trivial this class.

\begin{proposition}
    \label{prop:cc:ub:breach}
    \pbBReachCC is in LOGSPACE.
\end{proposition}
\begin{proof}
    From Lemma \ref{lemma:boundedUSTCONNinLOGSPACE}, one can construct an
    algorithm in LOGSPACE for \pbBReachCC. Indeed, given a configuration $c$ and
    $\ell\in\mathds{N}$, the straightforward iteration on the locations $c_i$
    followed by the verification of a path of at most $\ell$ (given in unary) steps from
    $\basenode$ to $c_i$ yields a sound and complete algorithm for \pbBReachCC.
\end{proof}

%\subsection{Lower bounds}

Our NP lower bound proof of the \pbBCoverageCC problem is by reduction from
the grid Hamiltonian cycle (G-HC) problem which is the Hamiltonian cycle problem
restricted to grid graphs and is NP-complete~\cite{Itai:1982}.
%\todo{What is a grid graph?}

% \begin{definition}[G-HC] The grid Hamiltonian cycle problem is the following
%   decision problem:
% %  \begin{itemize}
% %  \item[] Input:
%   Given a grid graph $G=\langle V,E \rangle$,
% %  \item[] Output:
%     decide if there is a simple cycle $t$ in $G$ which contains all vertices?
%     %such that
%     %all $v\in V$ are seen only once in $t$?
% %    does there exist a tour $t$ in $G$, that is a path such that
% %    all $v\in V$ are seen only once in $t$?
%   \end{itemize}
% \end{definition}

%\begin{theorem}
%  G-HC problem is NP-complete \cite{Itai:1982}.
%\end{theorem}
%Our result is by reduction from G-HC:
\begin{theorem} \label{prop:cc:lb:bcover}
  \label{th:cc:compl:bcover}
  \pbBCoverageCC is NP-complete.
\end{theorem}
\iffull
\begin{proof}
  The upper bound follows from Proposition~\ref{prop:dir:ub:bcover-breach}.
  
  We give a polynomial-time reduction from the G-HC problem.
  Consider a graph~$G=\langle V, E\rangle$, an instance of~G-HC.
  %We describe
  %the construction of an instance of \pbBCoverageCC from a G-HC instance below.

  Consider the sight-moveable topological graph %\linebreak[4] 
  $G'=\langle V,
  \moves, \coms\rangle$ with $\moves~=E$ and $\coms=V\times V$ and associate a
  single agent and the bound~$|V|$ to the \pbBCoverageCC instance.
  We call a simple cycle containing all vertices a \emph{tour}.
  We prove that there exists a tour $t$ in $G$ iff there exists a
  covering execution of length $|V|$ in $G'$.

  ($\Rightarrow$) Any tour of~$G$ is a valid execution satisfying
  \pbBCoverageCC since the communication edges form a complete graph,
  and the bound is~$|V|$.

  ($\Leftarrow$) Let us suppose that we have an execution of length $|V|$
  which covers the graph $G'$. The execution starts and ends at $\basenode$
  and
  visits all nodes in $|V|$ steps. Hence, the execution visits all nodes only
  once and is a cycle in the graph.

\end{proof}
\else
The upper bound follows from Proposition~\ref{prop:dir:ub:bcover-breach}. The
  NP-hardness proof is given in the long version.
\fi

% \begin{theorem}
%   \label{th:cc:compl:bcover}
%   \pbBCoverageCC is NP-complete.
% \end{theorem}

%%% Local Variables:
%%% mode: latex
%%% TeX-master: "main"
%%% End:

\section{Related Work}
\label{sec:relate}
\iffull
\begin{figure}[!t]
  \centering
  {\footnotesize
      \begin{tabular}{l|c|c|c|c}
        & $Reach$ & $Cover$ & $bReach$ & $bCover$
        \\\hline
        
        Directed & PSPACE-c & PSPACE-c &
        \multirow{3}{*}{\shortstack{NP-c\\ {[Hollinger]}}} &%\nocite{Hollinger:2012}}} &
        \multirow{5}{*}{NP-c} \\
        \cline{1-3}

        NC & \multirow{2}{*}{\shortstack{PSPACE-c\\ {[Tateo et al.]}}} %\nocite{dblp:conf/aaai/tateobrab18}}}
        & PSPACE-c & &
        % \\
        % \begin{tikzpicture}
        %   \node (s) at (0, 0) {};
        %   \node (t) at (3, 0) {};
        %   \draw[->] (s) -- (t);
        %   \draw[communication,tobehere] (s) edge[bend right=20] (t);
        % \end{tikzpicture} & & & & 
        \\
        % &&&&\\
        \cline{1-1}\cline{3-3}

        Undirected & & ? & & % \\  
        % \begin{tikzpicture}
        %   \node (s) at (0, 0) {};
        %   \node (t) at (3, 0) {};
        %   \draw[-] (s) -- (t);
        %   \draw[communication,tobehere] (s) edge[bend right=20] (t);
        % \end{tikzpicture} & & & &
        \\\cline{1-4}
        
        SM &
        \multirow{2}{*}{in L} &
        \multirow{2}{*}{in L} &
        NP-c &
        % \\ 
        % \begin{tikzpicture}
        %   \node (s) at (0, 0) {};
        %   \node (t) at (3, 0) {};
        %   \node (1) at (1, -0.5) {};
        %   \node (2) at (2, -0.5) {};
        %   \draw[communication] (s) -- (t);
        %   \draw[-, tobehere] (s) -- (1);
        %   \draw[-, tobehere] (1) -- (2);
        %   \draw[-, tobehere] (2) -- (t);
        %   \draw[communication, tobehere] (s) edge[bend right=10] (1);
        %   \draw[communication, tobehere] (1) edge[bend right=10] (2);
        %   \draw[communication, tobehere] (2) edge[bend right=10] (t);
        %   \draw[communication, tobehere] (1) edge[bend right=4] (t);
        %   \draw[communication, tobehere] (s) edge[bend right=4] (2);
        % \end{tikzpicture} & & &$(k=3)$  &
        \\
        \cline{1-1}\cline{4-4}

        CM & & &
        in L & \\
\end{tabular}
  }
  \caption{Complexity results.}
  \label{fig:results}
\end{figure}
\else
\input{results_tab_small}
\fi
The coverage planning is an interesting approach to path planning. Indeed, a
covering plan can be used for fields such as floor cleaning, lawn mowing, etc. A
survey of this field appears in \cite{Choset:01}. This multi-agent extension has
the ability to reduce the length of the overall mission and also reach parts of
the area a single agent would not able to. This problem was studied in
\cite{Rekleitis:97} for two agents. As shown in the survey by Chen et al.
\cite{ChenSurvey}, many coverage problems have been addressed by using analytic
techniques. For instance, in \cite{DBLP:conf/icc/Yanmaz12} and
\cite{teacy2010maintaining}, they consider UAVs that should cover an area
while staying connected to the base, but only empirically study some path planning
algorithms without proving their algorithms formally.
%planning algorithms and their algorithms are not proven formally but only
%tested experimentally.

We advocate formal methods that give formal guarantees and have already been applied to
generate plans for robots and UAVs. Model checking has been
applied to robot planning (see \cite{DBLP:conf/iros/LacerdaPH14}) and to UAVs~\cite{webster2011formal}.
Humphrey \cite{Humphrey2013} shows how to use LTL (linear-temporal logic) model
checking for capturing response and fairness properties in cooperation (for
instance, if a task is requested then it is eventually performed). 
%Model
%checking has also been used to verify pre-programmed UAVs
%\cite{webster2011formal}.

% In \cite{DBLP:journals/corr/abs-1003-0381}, they discuss CTL model checking
% for checking properties. CTL is not suitable for our purpose because we need
% to express the existence of \emph{one} path along which UAVs stay connected
% and eventually have covered all the locations and have come back to the base
% location.

Bodin et al. \cite{IJCAI2018demodrones} treat a similar problem except that the
UAVs cover the graph without returning to the base. Without the return-to-the-base
constraint, we claim that all our hardness results still hold, except
for \pbBCoverageCC. They provide an implementation by describing the problem in
Planning Domain Description Language and then run the planner
Functional Strips \cite{DBLP:conf/ijcai/FrancesRLG17}.

% Both \pbReach{} and \pbCoverage{} may be expressed in MA-STRIPS
% \cite{DBLP:conf/aips/BrafmanD08}, that is a multi-agent variant of STRIPS
% (Stanford Research Institute Problem Solver) in which actions for each agent
% can be described independently. The representation in multi-agent planning
% languages is especially efficient when actions of the different agents are
% independent and when they required to coordinate not so often. However, as the
% agents should maintain connection, it requires a lot of coordination.

Murano et al. \cite{DBLP:conf/prima/MuranoPR15} advocate for a
graph-theoretic representations of states, that is, by assigning locations to
agents as in Definition~\ref{def:config}. In
\cite{DBLP:conf/atal/AminofMRZ16,DBLP:conf/atal/Rubin15}, a general formalism is given to specify LTL and monadic second-order logic properties,
which are expressive enough to describe the connectivity constraint.
%Indeed, linear temporal operators enable to express that any vertex
%should be visited in the future and the connectivity invariant. MSO on the
%topological graph enables to express the connectivity as a fix point (the
%subgraph made up of the UAVs and the base is connected).
They provide an
algorithm for parametrized verification in the sense that they check a temporal
property in a class of graphs. This is relevant for partially-known
environments. The algorithm described is
non-elementary %\footnote{\todo[inline]{explain non-elementary}}
(\textit{i.e.} the running time cannot bounded by any tower of exponentials)
and therefore not
usable in practice. We believe that this is an important problem
and our paper identifies an efficient and relevant fragment.
%studying fragments of this is
%relevant, and our paper identifies a relevant fragment.

The multiple traveling salesman problem (mTSP) is a generalization of the
traveling salesman problem (TSP) in which multiple salesmen are located at a
depot \cite{Anbuudayasankar:2016}. mTSP asks for the coverage of all cities so as to minimize the total plan cost by visiting each city exactly once.
%This generalization of TSP is considered as a
%relaxation of the vehicle routing problem (VRP) in which the capacity constraint
%is removed.
An overview of TSP and its extensions are presented
\cite{Matai:2010}. The \pbCoverage{} problem is related to mTSP, since we use
results on Hamiltonian cycle to prove the NP-hardness of \pbBCoverageCC.
However, we wish to minimize the length of the execution and not the cost of the
execution. Those problems are equivalent on unit graphs, but it is not trivial
to use general results on mTSP in order to solve \pbCoverage{}. Furthermore, to
the best of our knowledge, connected versions of mTSP and VRP have not been
studied.

%%% Local Variables:
%%% mode: latex
%%% TeX-master: "main"
%%% End:

\section{Conclusion}
\label{sec:concl}
% In this work, we extended the study of multi agent connected path planning 
% by considering the coverage problem. We showed that in the general case the complexity of the
% decision of the coverage matches the complexity of the reachability problem.
% Furthermore, this complexity still holds for the neighbor-communicable
% subclass.
% We identified an important subclass of topological graphs on which the complexity
% of the problem is as low as LOGSPACE.
% Unfortunately, the bounded
% versions of both problems stays NP-complete.
% A LOGSPACE algorithm can be obtained by ignoring communication constraints,
% that is, for the complete-communication subclass.
%However, an even more restrictive
%subclass admits a logarithmic-space algorithm for the bounded reachability
%problem.

%We can observe that the
Sight-moveable topological graphs we introduced in this work only constrain
the communication graph. One can be interested to constrain the movement graph
to a planar graph or a 2D grid given the common usage of grid modelling of the
environment. Given the intractability of MAPP on planar graphs \cite{Yu:2015}
and on general 2D grid graphs \cite{Banfi:2017}, it is likely that this
problem is intractable as well. Furthermore, in
\cite{dblp:conf/aaai/tateobrab18}, the decision is proved to stay
PSPACE-complete on planar graphs and grids as well. However, one can study this
problem on solid grid graphs, given that the Hamiltonian cycle is tractable on
such graphs \cite{Umans:1997}.

One can note that our NP lower bound reductions hold without the
anonymity of the agents. Indeed, the \pbBCoverage{} case is straightforward and for \pbBReach{} case, each agent can be associated to a clause or variable, so the reduction would still hold.
%Thus, the lower bound holds without anonymity.
%\todo{Ocan: Do we claim in the sense of conjecture or do we state it? Please avoid using claim in the latter sense}

We do not know  if \pbCoverage{} remains hard when the
$\rightarrow$-relations become symmetric, depicted in Figure~\ref{fig:results}
as a question mark. We think this open issue is important since symmetric
{$\rightarrow$-relations} (if UAVs can go from $v$ to $v'$, they can also come
back from $v'$ to $v$) are relevant for practical applications. We plan to
study the \emph{parametrized complexity} \cite{DBLP:series/mcs/DowneyF99} of our
problems - parameters could be the treewidth of the topological graph, the
number of UAVs.

% \todo{Can we remove following two paragraphs?}
% We attend to develop efficient parallel algorithms for the problems proven in
% LOGSPACE. Those algorithms can yield feasible solution in a short amount of time
% in order to be improved by heuristic-based algorithms. The sight-moveable
% topological graphs being a ``simpler'' class than the undirected ones, we can
% hope to obtain a polynomial-time approximation algorithm. Furthermore, this
% class seems to admit an intuitive resolution process which involve finding the
% relay positions and optimizing the coverage of multiple small areas.

% Interestingly, we plan to generalize to decentralized versions of our problems
% and to dynamic environments. Instead of generating sequences of actions, we will
% have to generate strategies as in ATL (alternating-time temporal logic)
% \cite{DBLP:journals/corr/abs-1006-1414}. As UAVs stay connected, we may suppose
% that when information is gained, it is common knowledge and that all actions,
% especially sensing actions, are public \cite{DBLP:conf/atal/BelardinelliLMR17}.
% We also aim at using a high-level dedicated formal logic to express objectives,
% such as the language proposed in \cite{DBLP:conf/atal/Rubin15} and
% \cite{DBLP:conf/atal/AminofMRZ16}.

%%% Local Variables:
%%% mode: latex
%%% TeX-master: "main"
%%% End:

%%% Local Variables:
%%% mode: latex
%%% TeX-master: "main"
%%% End:

\appendix

\bibliographystyle{named}
\bibliography{biblio}

\end{document}